\documentclass[runningheads]{llncs}
\usepackage{etoolbox}
\usepackage{amsmath}
\usepackage{gensymb}
\usepackage{breqn}
\usepackage{caption}
\usepackage{graphicx}
\usepackage{cleveref}
\usepackage{algorithm}
\usepackage{algorithmic}
\usepackage{subfigure}
\usepackage{amssymb}
\DeclareMathOperator*{\argmax}{argmax}
\DeclareMathOperator*{\argmin}{argmin}
\DeclareMathOperator{\EX}{\mathbb{E}}%
\usepackage{bigints}
\usepackage{longtable}
\usepackage{booktabs}
\usepackage{color}
\usepackage{hyperref}

\usepackage{fancyhdr}

\usepackage{todonotes}

\begin{document}

\title{Designing MacPherson Suspension Architectures using Bayesian Optimization\thanks{Supported by Ford through the Ford-KU Leuven Alliance.}
\thanks{Copyright \textcopyright 2019 for this paper by its authors. Use permitted under Creative Commons License Attribution 4.0 International (CC BY 4.0).}}
\titlerunning{Designing MacPherson Suspension Architectures}


\author{Sinnu Susan Thomas\inst{1}\and
Jacopo Palandri\inst{2} \and
Mohsen Lakehal-ayat\inst{2}\and
Punarjay Chakravarty\inst{3}\and
Friedrich Wolf-Monheim\inst{2}\and
Matthew B.\ Blaschko\inst{1}}
\authorrunning{Thomas et al.}

\institute{ESAT-PSI, KU Leuven, Kasteelpark Arenberg 10, 3001 Leuven, Belgium\\ 
\email{sinnu.thomas,matthew.blaschko@kuleuven.be}\\
\url{https://www.esat.kuleuven.be/psi} \and Ford Research \& Innovation Center, S{\"u}sterfeldstra{\ss}e 200, 52072 Aachen, Germany
\email{jpalandr, mlakehal,fwolf5@ford.com}\\
\url{https://www.ford.de/}\and Ford Greenfield Labs, 3251 Hillview Ave, Palo Alto, CA 94304, USA
\email{pchakra5@ford.com}\\
\url{https://www.ford.com/}}
\maketitle              
\begin{abstract}
Engineering design is traditionally performed by hand: an expert makes design proposals based on past experience, and these proposals are then tested for compliance with certain target specifications. Testing for compliance is performed first by computer simulation using what is called a discipline model. Such a model can be implemented by a finite element analysis, multibody systems approach, etc. Designs passing this simulation are then considered for physical prototyping. The overall process may take months, and is a significant cost in practice. We have developed a Bayesian optimization system for partially automating this process by directly optimizing compliance with the target specification with respect to the design parameters. The proposed method is a general framework for computing a generalized inverse of a high-dimensional non-linear function that does not require e.g.\ gradient information, which is often unavailable from discipline models.  We furthermore develop a two-tier convergence criterion based on (i) convergence to a solution optimally satisfying all specified design criteria, or (ii)  
convergence to a minimum-norm solution. We demonstrate the proposed approach on a vehicle chassis design problem motivated by an industry setting using a state-of-the-art commercial discipline model. We show that the proposed approach is general, scalable, and efficient, and that the novel convergence criteria can be implemented straightforwardly based on existing concepts and subroutines in popular Bayesian optimization software packages.

\end{abstract}

\section{Introduction}
The handling stability, running quality, and vehicle ride comfort is the criteria used by the passenger for assessing on-road vehicles. The handling attributes are determined by the way the forces at the tire contact patch influence the vehicle dynamics.
The suspension is a complex multi-body system that connects and transmits the force between the vehicle body and the wheel. 
The forces can be due to a driver command such as a steering wheel, gas pedal or brake input, but also to road surface unevenness, aerodynamic forces, vibrations of the engine and drive-line, and non-uniformity of the tire/wheel assembly. 

The desired behavior of the vehicle can be achieved by optimizing the design of the front suspension, which comprises a large number of design variables. One of the most popular suspension systems is the MacPherson strut \cite{macpherson1953vehicle,macpherson1953wheel}, which is deployed in a large number of vehicles for its simple structure, good overall performance, package efficiency and relatively low cost as shown in Fig.\ref{suspension model}.
\begin{figure}
    \centering
    \includegraphics[trim={6cm 4cm 6cm 2.2cm},width=6.8cm,height=4.2cm]{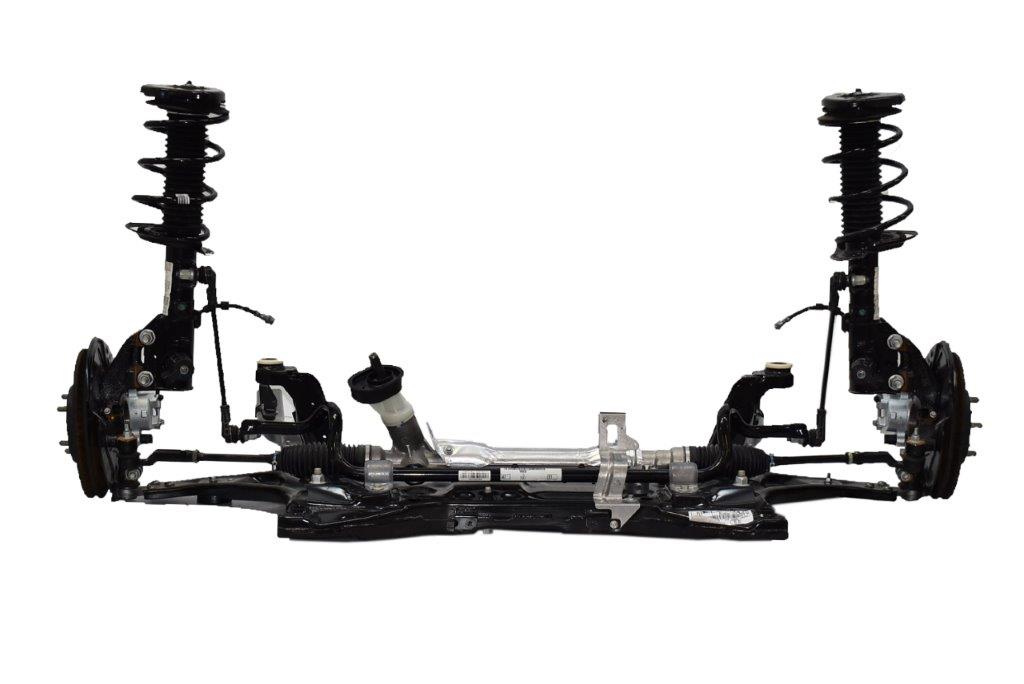}
    \caption{MacPherson Suspension Design \cite{macpherson1953vehicle,macpherson1953wheel}}
    \label{suspension model}
\end{figure}

In this work, we develop a Bayesian optimization strategy for determining parameters of a MacPherson suspension system that optimizes desired performance characteristics of an automobile.  In contrast to the industry standard of hand designed parameters, we aim to achieve partial automation of the design process by employing a multibody dynamics simulation model and allowing the Bayesian optimization to iteratively explore the design space without human intervention.

\subsection{Applications of Bayesian Optimization}
Inverse problems are generally the problem of finding the properties of the model using indirect measurements such as finding an inverse shortest path and inverse minimum spanning tree problems \cite{Yang1999AML}. Vito et al.\ \cite{Vito2005JML} formulated the problem of finding the best solution as a linear inverse problem given some loss function and hypothesis space. This problem can also be solved using the regularized least squares algorithm. These types of problems are ill-posed problems \cite{Ye2019MA}. They are often solved using Bayesian approaches. In the proposed approach, the design of the suspension of the vehicle is framed as the inverse problem and the inverse problem is solved using Bayesian Optimization.

Bayesian Optimization (BO) creates a surrogate model of the black-box objective function and finds the optimum of an unknown objective function from which samples can be obtained \cite{Mockus:1974:BMS:646296.687872,BrochuarXiv2010}. BO has been studied in many fields for the optimization of unknown functions. 
The tuning of the hyper-parameters,  regularization terms, and optimization parameters in machine learning needs a lot of expert experience, but with the help of the tractable posterior  distribution  induced  by  the  Gaussian Processes \cite{Snoek2012NIPS} leads  to  the efficient  use  of the information gathered from the past experiences, enabling optimal choices for the next parameters. 
The performance of the BO is highly dependent on the choice of acquisition function made to find the best-observed value.  \cite{pmlr-v77-nguyen17a,Ma2019a} have proposed convergence criteria for several acquisition functions in order to avoid unwanted evaluations. 

Several methods exist in the literature for designing suspension architectures. Yu and Yu \cite{YuPSU2003}  formulated the problem of optimal vehicle suspension design with the quarter-car vehicle dynamic model using a genetic algorithm (GA). 
The two objectives of the optimization are to minimize the average suspension displacement and to minimize the maximum bouncing acceleration of the sprung mass \cite{Blundell2014} restricted to a number of constraints. An objective function is formulated as the sum of absolute magnitude of sprung mass  (chassis) acceleration and absolute magnitude of average sprung mass displacement. The authors have considered some constraints in the problem formulation such as the maximum amplitude of sprung mass acceleration should not be less than 1g (9.8m/$s^2$), human beings feel motion sickness when vibration at a natural frequency is less than 1Hz, and the absolute value of the relative displacement between the sprung and unsprung mass should be less than 13 cm. Zhang et al.\ \cite{ZhangESDA2006} proposed the design of suspension parameters of flexible multibody vehicle model using the ADAMS software and GA. A flexible multibody model having 44 degrees of freedom  is compared with the multi-rigid bodies having 33 degrees of freedom. ADAMS is a tool used by the leading manufacturers to design parameters for complex design and optimization. ADAMS and GA carries the simulation analysis until the algorithm converges. Yao et al.\ \cite{YaoAMM2013} modeled a sedan MacPherson front suspension with ADAMS and studied the relationship between the wheel alignment parameters and tire wear when wheel hops.
Afkar et al.\ \cite{AfkarJV2012} modeled a double wishbone suspension system using ADAMS and then optimized geometric parameters using a GA. Similarly, the authors in \cite{QianICICA2012} optimized several hard points of the front suspension using the INSIGHT module of ADAMS to optimize toe and camber angle characteristics. Dye and Lankarani \cite{DyeIDETC2016} used a neural network to fit tire test data under varying tire pressures and steady state conditions. Hurel et al.\ \cite{HurelAMC2012} developed a non-linear two-dimensional mathematical quarter-car model of MacPherson suspension. This model considered both the vertical motion of the chassis and the rotation and translation for the unsprung mass (wheel assembly). Zhang et al.\ \cite{ZhangICVMEE2016} studied the geometric relationship between the structural parameter of the double wishbone front suspension and alignment parameters of the steering wheel. Liu et al.\ \cite{LiuRJASET2013} studied the hardpoints of the double wishbone suspension of a Formula SAE car using correlation theory. Tao et al.\ \cite{TaoIDETC2017} used a Gaussian process modeling method to design vehicle suspension parameters using a multi-disciplinary optimization architecture, and experiments were carried out for a front MacPherson suspension and a rear strut suspension system on a Altair Motion View vehicle model. The Gaussian process models are fitted with multidisciplinary optimization and efficiencies are computed.

In automobile manufacturing, whether ride and handling characteristics have been achieved is quantified through statistics of a kinematic curve. The design process can be time intensive as an engineer has to manually input a particular set of hard-points, pass it through the simulation tool (which can take several minutes to process), observe the kinematic curves, and then repeat the whole process to refine the design. Thus, generating one design that meets the desired kinematic constraints in current standard practice can take multiple weeks and consume significant manpower. 

The main novel contributions of this paper are as follows:
\begin{itemize}
    \item Bayesian optimization is applied to the design of the front suspension of a vehicle.
    \item Finding the geometric locations of the hardpoints of the suspension from the desired target characteristics is formulated as a generalized-inverse problem.
    \item We develop novel convergence criteria in the context of the generalized-inverse problem. 
    \item We verify the validity of the approach in 
    several
    suspension design settings, empirically validating the predicted convergence characteristics.
\end{itemize} 

\section{Suspension Design using Bayesian Optimization}
\subsection{Bayesian Optimization}
BO is a derivative free optimization approach for global optimization of an expensive black-box function $f$. It is a class of sequential-model based optimization algorithms that uses past evaluations of the function to find the next point to sample.  For a minimization problem,
\begin{equation}
        \mathbf{y}^*\; =\;\arg\min_{\mathbf{y}\in \mathcal{Y}}\;f(\mathbf{y}),
    \end{equation}
if $f$ is expensive, then there is a strong need to reduce the number of evaluations. Since the objective function is unknown, according to Bayesian theory the function is treated as random and a prior belief is placed over the function values, and as more values are observed, the posterior is updated based on the observation likelihood.  
These models use different types of acquisition functions to determine the next point to be evaluated based on the current posterior distribution over functions. 

The surrogate model used for this optimization is a Gaussian Processes (GP). A GP is characterized by its mean $\mu(y)$, and co-variance function $k (y,y^{'})$. For $n$ data points, the function $f_{1;n}\;=\; f (y_{1}), \;\ldots,\; f (y_{n})$ can be characterized using a multivariate Gaussian distribution as
\begin{equation}
    f_{1:n}\;=\; \mathcal{N}\;(\mu\;(\mathbf{y}_{1:n}),\;\mathbf{K}),
\end{equation}
where $\mathbf{K}$ is a $n\times n$ kernel matrix given by
\begin{equation}
\mathbf{K}\;=\; \begin{pmatrix}
k(\mathbf{y}_1,\mathbf{y}_1) & \;\ldots & \; k(\mathbf{y}_1,\mathbf{y}_n) \\
\;\vdots & \;\ddots & \;\vdots\\
\;k(\mathbf{y}_n,\mathbf{y}_1) & \;\ldots & \;k(\mathbf{y}_n,\mathbf{y}_n)
\end{pmatrix},
\end{equation}
for some positive definite kernel such as a Gaussian or Matern kernel \cite{Rasmussen:2005:GPM:1162254}.

An acquisition function proposes which points should be selected next to determine the minimizer of the function, and they trade off between exploitation and exploration. Exploitation means the acquisition function selects points where the mean of the GP is low and exploration means the acquisition function selects points where the variance of the GP is high. Various acquisition functions are used for Bayesian models, including maximum probability of improvement (MPI), expected improvement (EI), confidence bounds (CB) criteria.

Maximum Probability of Improvement evaluates $f (y)$ at the point most likely to improve than the minimal value of $f (y^{+})$ observed so far \cite{BrochuarXiv2010}
\begin{align} 
MPI(\mathbf{y})\; &=\; P(f(\mathbf{y})\;\geq\; f(\mathbf{y}^{+})\; +\; \xi), \nonumber \\
 &=\;  \Phi \left( \frac{\mu(\mathbf{y})\;-\; f(\mathbf{y}^{+})\;-\;\xi}{\sigma(\mathbf{y})} \right),
\end{align}
where $f(y^{+})$ is the best observed value of the function so far, $\mu(\mathbf{y})$ is the posterior mean of $\mathbf{y}$ under the GP, $\sigma(\mathbf{y})$ is the posterior standard deviation, and $\xi$ is a parameter used to drive exploration, which is usually very small. 
$\Phi$ is the cumulative distribution function of a standard normal variable.

Expected Improvement evaluates $f(y)$ at the point where the expectation of the improvement in $f(y^{+})$ under the current estimate of the GP is the highest \cite{PradeepSDSU}
\begin{align} 
EI(\mathbf{y})\; &=\; \EX\;[\max\{0\;,\; f(\mathbf{y}^{+}) \;-\;f(\mathbf{y})\}], \nonumber \\
 &=\;  (-\mu(\mathbf{y})\;+\; f(\mathbf{\mathbf{y}}^{+})\;+\;\xi)\; \Phi \left( \frac{-\mu(\mathbf{y})\;+\; f(\mathbf{y}^{+})\;+\;\xi}{\sigma(\mathbf{y})} \right)\; + \nonumber \\
 &\sigma(\mathbf{y})\; \phi \left( \frac{-\mu(\mathbf{y})\;+\; f(\mathbf{y}^{+})\;+\;\xi}{\sigma(\mathbf{y})} \right) , \label{eq:ExpectedImprovementAcquisition}
\end{align}
where $\phi$ is the probability density function of a standard normal variable.

The basic algorithm of Bayesian Optimization \cite{NoearXiv2018} is given in Algorithm \ref{alg1}.
\begin{algorithm}[tb]
   \caption{Bayesian Optimization}
   \label{alg1}
\begin{algorithmic}
   \STATE {\bfseries Initial Design:} \\ \hspace{1.5em}$\mathbb{D}_{n_{init}}$= $\{(\mathbf{y}_i, x_i)\}_{i=1}^{n_{init}}$,\\
   \hspace{1.5em}$n_{max}$ function evaluation
   \STATE {\bfseries Results:} \\
   \hspace{1.5em}Estimated min:$f_{min} = \min(x_1,\ldots, x_{n_{max}})$.
   \STATE \hspace{1.3em} Estimated minimum point:
   \begin{equation}
       \mathbf{y}_{min} =  {\mathrm{argmin}}(x_1,\ldots,x_{n_{max}}).
   \end{equation}
   \FOR{$n = n_{init}$ {\bfseries to} $n_{max}$}
   \STATE Update GP: $f(\mathbf{y}) | \mathbb{D}_n \sim GP(\hat{f}(\mathbf{y}), \mathbb{K}(\mathbf{y},\mathbf{y^'}))$,
   \STATE Optimize acquisition function: \[\mathbf{y}_{next} = \argmax_{\mathbf{y} \in \mathcal{Y}} \alpha_n(\mathbf{y}).\]
   \STATE Find $f$ at $\mathbf{y_{next}}$ to obtain $x_{next}$.
   \STATE Add data to previous design: $\mathbb{D}_{n+1} = \mathbb{D}_n \cup \{\mathbf{y}_{next},x_{next}\}$.
   \ENDFOR
\end{algorithmic}
\end{algorithm}

\subsection{Generalized Inverse Problems using Bayesian Optimization}
Automotive engineers use mechanical discipline models to design the suspension of the vehicle, which is computationally expensive and may involve a design cycle spanning several days or weeks. Operating characteristics of the vehicle suspension components depend upon the suspension architecture, selection of components and the locations at which the components are attached to each other and the vehicle body. The suspension geometry includes the arrangement of suspension components and the location of hard-points, which are locations at which suspension components are attached to the vehicle's body. For a given set of vehicle suspension components, the suspension geometry will determine the kinematic characteristics of the vehicle suspension which  includes wheel attitude and suspension travel. Kinematic characteristics are traditionally derived from a multi-body dynamic modeling/simulation of an entire suspension system subjected to spindle input loads such as may be performed by commercially available multi-body dynamic (MBD) software. Given a set of hard-point locations and corresponding kinematic curves we learn the relationship between the hard-point locations and kinematic curves using BO. This work predicts a set of kinematic curves corresponding to hard-points for MacPherson strut suspension. Predicting the kinematic curves corresponding to hard-points can reduce the amount of time required to select hard-points for a MacPherson strut suspension design process from multiple weeks to less than one day.

Statistics of kinematic curves of the vehicle form our representation of the desired vehicle characteristics. We denote the set of possible statistics as $\mathcal{X}$ where $\mathcal{X} \subseteq \mathbb{R}^d$.  This paper addresses the design of suspension hardpoints $\mathcal{Y}$ where $\mathcal{Y} \subset \mathbb{R}^m$ is a bounded domain for a given target $\mathbf{x} \in \mathcal{X}$. It is a well defined task for the engineers to design target characteristics $\mathcal{X}$ from the suspension design parameters $\mathcal{Y}$. 
Consider the equation   
\begin{equation}
    \mathbf{x}\;=\; g(\mathbf{y}).
\end{equation}
We formulated this problem as an inverse problem for $g : \mathcal{Y} \rightarrow \mathcal{X} $, find a value $\mathbf{y}$ such that $g(\mathbf{y})\approx \mathbf{x}$. 

In general the above problem is ill-posed: there may be multiple $\mathbf{y}$ satisfying $g(\mathbf{y})=\mathbf{x}$ for a given $\mathbf{x}$ or none. The existence  of the solution can be restored by solving for the minimum norm solution 
\begin{equation}
  g^{\dagger}(\mathbf{x})\;:=\; \arg\min_{\mathbf{y} \in \mathcal{Y}}\; \underbrace{\|g(\mathbf{y})-\mathbf{x}\|^2}_{=:f(\mathbf{y})} , \label{eq12}
\end{equation}
where $g^{\dagger}$ defines a generalized inverse of $g$ \cite{ben2006generalized}.
We optimize the desired suspension design parameters using Bayesian Optimization (BO) where $g$ is set to be computed by the multi-body dynamic software MSC ADAMS \cite{MSC-ADAMS}, and Bayesian optimization is used to minimize $f$ with respect to $\mathbf{y}$.

\subsection{Convergence criteria}

It is well known that Bayesian optimization converges under fairly general conditions to a global optimizer over a bounded domain \cite{pmlr-v77-nguyen17a,BrochuarXiv2010,Bull:2011:CRE:1953048.2078198}.
That the algorithm will converge towards an optimum indicates that, when the requested design characteristics $\mathbf{x}$ are feasible, i.e.\ 
\begin{align}
    \arg\min_{\mathbf{y}\in\mathcal{Y}} \;\|g(\mathbf{y}) - \mathbf{x}\|^2 \; =\; 0 ,
\end{align}
we may simply specify a convergence criterion that terminates when
\begin{equation}
   \|g(\mathbf{y}^+) - \mathbf{x}\|^2\; <\; \varepsilon, 
\end{equation}
for some prespecified tolerance $\varepsilon\;>\;0$.  However, it may not be the case that the requested design characteristics $\mathbf{x}$ lie in the image of $\mathcal{Y}$ under $g$.  Consider the example where passenger comfort and acceleration must simultaneously be unrealistically high: such constraints can contradict each other.  In such a case it may be that for a given tolerance, 
\begin{align}
\arg\min_{\mathbf{y}\in\mathcal{Y}} \;\|g(\mathbf{y}) - \mathbf{x}\|^2\; >\;\varepsilon,
\end{align}
and the optimization will never terminate.  We therefore develop an additional termination criterion in the sequel.

Bayesian optimization is typically presented without termination criteria, or with the assumption that there is a fixed budget of optimization iterations \cite{DBLP:journals/corr/abs-1807-02811}.  However, two main optimization dependent termination strategies have been proposed: (i) thresholding the acquisition function \cite{pmlr-v77-nguyen17a,Ma2019a}, and (ii) terminating when two sequential parameter vectors differ by a small amount \cite{lorenz2015stopping}.  We consider here the strategy of thresholding the acquisition function, as this has been analyzed from the perspective of regret minimization, which implies a bound on the convergence rate of the optimization \cite{pmlr-v77-nguyen17a,SrinivasarXiv2009}.

\begin{theorem}[\cite{Rasmussen:2005:GPM:1162254}, Sec.~2.9]\label{thm:ConvergenceofGP}
Let $\sigma_n(f(y))$ denote the predictive variance of a Gaussian process regression model at $y$ given a dataset of size $n$.  The predictive variance of using a dataset of only the first $n-1$ training points is denoted $\sigma_{n-1}(f(y))$:
\begin{equation}
\sigma_n(f(y))\; \leq\; \sigma_{n-1}(f(y)).
\end{equation}
\end{theorem}
Furthermore, for a covariance function that is strictly positive definite (i.e.\ the rank of the covariance matrix is equal to its size), the variance will be strictly decreasing over a bounded domain $\mathcal{Y}$.

It is well known that Bayesian optimization with expected improvement will converge to an optimum
\cite[Theorem~2]{Bull:2011:CRE:1953048.2078198}.  This leads us to the following proposition:
\begin{proposition}[Termination of Expected Improvement]
 The following convergence criterion applied to Algorithm~\ref{alg1} will terminate: 
\begin{align}
    \max_{\mathbf{y}\in\mathcal{Y}} \;EI(\mathbf{y})\; \leq\; \varepsilon ,
\end{align}
for a fixed $\varepsilon\;>\;0$.
\end{proposition}
\begin{proof}
From \cite[Theorem~2]{Bull:2011:CRE:1953048.2078198}, we have that 
\begin{align}
f(\mathbf{y}^+)\; -\; \arg\min_{\mathbf{y}\in \mathcal{Y}}\; f(\mathbf{y}),
\end{align}
is in expectation decreasing in the number of iterations, meaning that 
\begin{equation}
    -\mu(\mathbf(y))\;+\;f(\mathbf{y}^+)\; +\; \xi,
\end{equation}
is decreasing.  Furthermore, inspecting the formula for $EI(\mathbf{y})$, we see that $\Phi$ and $\phi$ are bounded.  From Theorem~\ref{thm:ConvergenceofGP}, we have that $\sigma(\mathbf{y})$ is decreasing, and the result follows. \hfill \qed
\end{proof}

When the maximum expected improvement is small, by definition, we believe under the posterior GP model that the solution will not improve by a large amount when observing a new point, 
and we are close to a minimum norm solution.  Furthermore, due to the quantification of uncertainty in the Gaussian process model, we may set a threshold on the acquisition function such that this second termination criterion  becomes active when it is highly probable that no solution with near-zero norm exists.

\section{Case Studies}

We implemented the inverse problem of finding the suspension design using BO \cite{gpyopt2016} in Python 3.7.3. The LBGFS algorithm with multiple restarts was used for optimizing the acquisition function. All the experiments were conducted on a workstation with two Intel Xeon CPUs and 64G memory. MSC ADAMS 18.1 was used as the discipline model, and takes several minutes per call.  
The problem becomes
\begin{equation}
    \mathbf{y}^*\; =\;  \argmin_{\mathbf{y} \in \mathcal{Y}} \;\|ADAMS(\mathbf{y})\;-\;\mathbf{x}\|^2,
\end{equation}
for a desired set of vehicle characteristics $\mathbf{x}$.
We optimize the positions of Inner Tie Rod Ball  and Outer Tie Rod Ball of the MacPherson front suspension architecture. 
We compare the convergence of Bayesian optimization using acquisition function EI to that of the commercial HEEDS Sherpa optimization software, which is commonly used in automotive design \cite{chase2012new,locci2019acoustic}.  

A classic suspension design problem is the tuning of the bump and roll steer performance, which is strictly related to the placement of the tie rod. Bump steer is a measure of the suspension toe angle with respect to suspension vertical travel and is measured in deg/m.  In roll steer the wheel travel is generated by applying a roll angle to the vehicle body. Tie rod connects the steering to the steering knuckle on each front wheel. A tie rod is made of two components: inner and outer tie rod ends. This problem is relevant as the toe attitude of the wheel is a function of its vertical motion and is related to crucial vehicle attributes such as vehicle stability, cornering performance, and steering feel. Indeed, since the position of the outer tie rod defines the lever arm to the kingpin axis, the outer tie rod location determines how the forces are transmitted from the contact patch through the suspension and ultimately to the driver operating the steering wheel. 

\subsection{Optimization of one hardpoint}

We optimized the position of Outer Tie Rod Ball as the first case study. This hardpoint has large effect on the performance of the bump steer and roll steer. We show the result of two targets such as bump steer and roll steer from the entire set of target characteristics. We achieved the convergence criterion of \begin{equation}
     \|ADAMS(\mathbf{y})-\mathbf{x}\|^2 \;\leq\; 0.001
\end{equation} in 210 iterations. 
The proposed design of a single hardpoint outperforms the design of the HEEDS software (Fig. \ref{Convergenceall1HP}).  The acquisition function converges as shown in Fig.~\ref{AcquisitionFunctionfor1HP}, and the norm per iteration is shown in Figs.~\ref{ConvergenceHEEDS1HP}-\ref{ConvergenceBO1HP}. We have also compared to a numerical gradient based optimization method using the MATLAB fmincon function (Fig. \ref{Convergencefmincon1HP}). Fig.~\ref{Convergenceall1HP} shows that the BO approach converges more quickly than the competing techniques.

Figs.~\ref{HEEDSScatterplotouter1HP}-\ref{BOScatterplotouter1HP}, show the points in the design space explored using HEEDS, fmincon, and Bayesian optimization. The axes of the scatter plots are normalized in the range $[0,1]$ with respect to the input domain of the hardpoints.

Fig.~\ref{Bumpsteer1HP} 
show the convergence of the optimized target characteristics towards the desired bump steer using HEEDS, fmincon, and Bayesian optimization. The optimized targets are normalized such that the target performance is labeled as zero and is shown with a red line. All three methods converge to similar targets, but fmincon converges with a non-zero error.
\begin{figure}[htp]
 \hspace*{-3.5em}
\begin{minipage}{.5\textwidth}
  \centering
  \includegraphics[width=5.8cm,height=4.2cm]{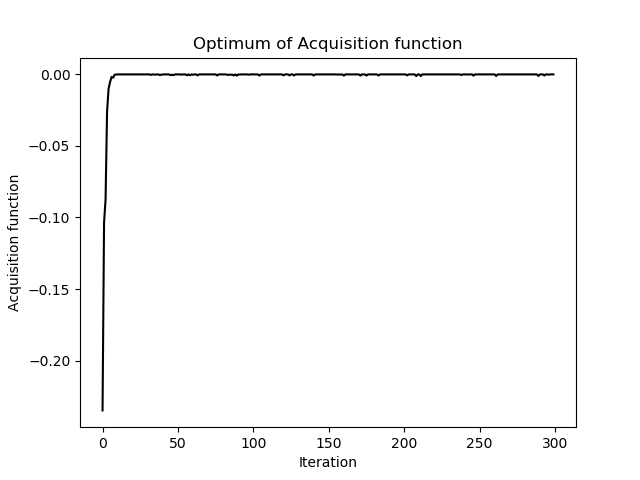}
  \captionof{figure}{Acquisition Function for 1 HP.}
  \label{AcquisitionFunctionfor1HP}
\end{minipage}%
\begin{minipage}{.5\textwidth}
  \centering
  \includegraphics[width=5.8cm,height=4.2cm]{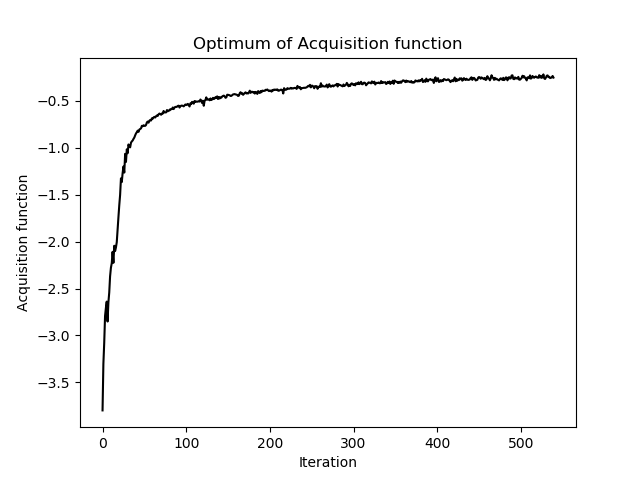}
  \captionof{figure}{Acquisition Function for 2 HP.}
  \label{AcquisitionFunctionfor2HP}
\end{minipage}
 \hspace*{-3.5em}
\begin{minipage}{.5\textwidth}
  \centering
  \includegraphics[width=5.8cm,height=4.2cm]{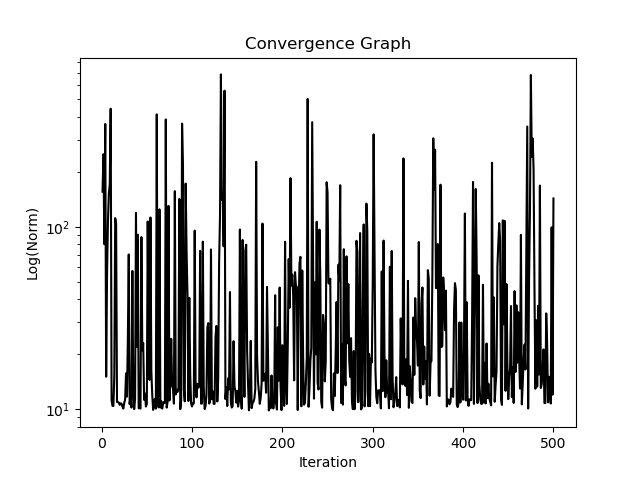}
  \captionof{figure}{Convergence (HEEDS) 1HP.}
  \label{ConvergenceHEEDS1HP}
\end{minipage}%
\begin{minipage}{.5\textwidth}
  \centering
  \includegraphics[width=5.8cm,height=4.2cm]{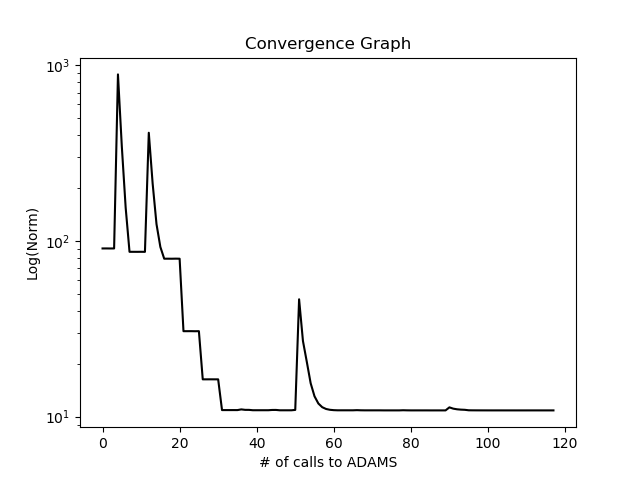}
  \captionof{figure}{Convergence (fmincon) 1HP}
  \label{Convergencefmincon1HP}
\end{minipage}
\begin{minipage}{.5\textwidth}
  \centering
  \includegraphics[width=5.8cm,height=4.2cm]{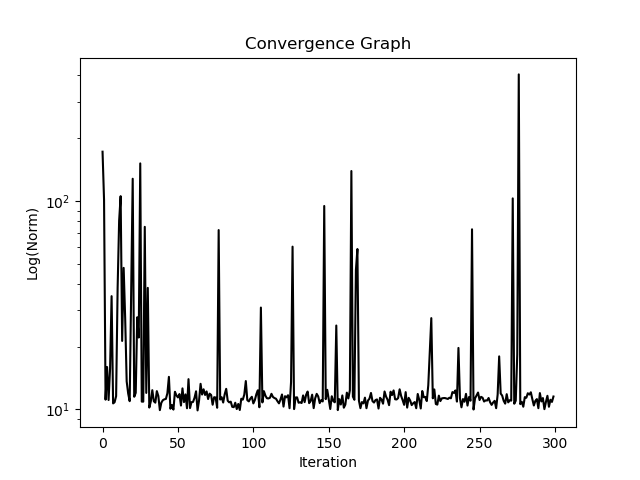}
  \captionof{figure}{Convergence (BO) 1HP.}
  \label{ConvergenceBO1HP}
\end{minipage}%
\begin{minipage}{.5\textwidth}
  \centering
  \includegraphics[width=5.8cm,height=4.2cm]{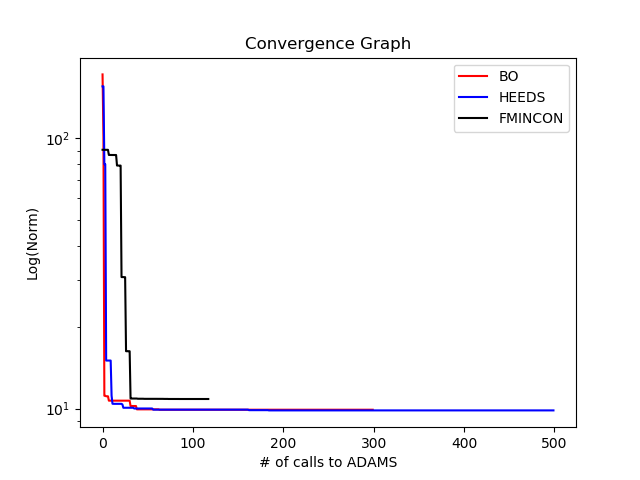}
  \captionof{figure}{Comparison of different convergence plots in 1HP case.}
  \label{Convergenceall1HP}
\end{minipage}
\hspace*{-5.1em}
\subfigure{
\includegraphics[width=4.9cm,height=4.2cm]{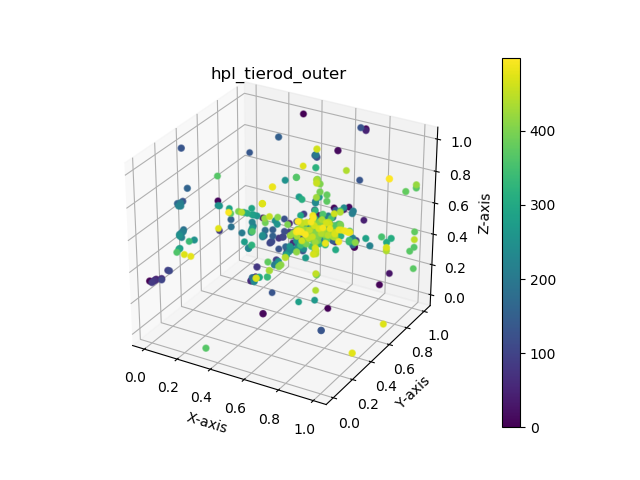}
\label{HEEDSScatterplotouter1HP}}
\hspace*{-3.1em}
\subfigure{
\includegraphics[width=4.9cm,height=4.2cm]{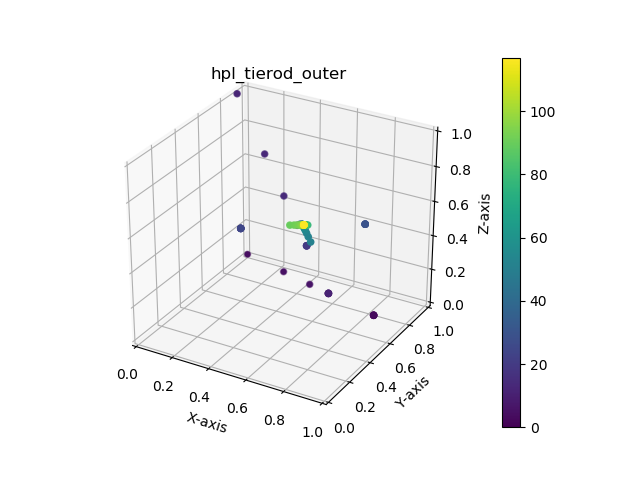}
\label{fminconScatterplotouter1HP}}
\hspace*{-3.1em}
\subfigure{
\includegraphics[width=4.9cm,height=4.2cm]{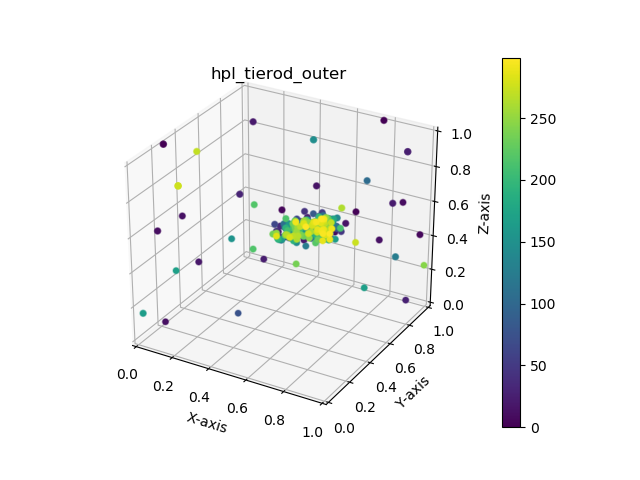}
\label{BOScatterplotouter1HP}}
\caption{Scatter plots for Outer Tie Rod Ball   in 1HP case (a) HEEDS.$\;$(b)nonlinear programming solver.$\;$(c) proposed approach.} 
\label{Scatterplotouter1HP}
\end{figure}
\begin{figure}[htp]
\hspace*{-3.1em}
\subfigure{
\includegraphics[width=4.6cm,height=4.2cm]{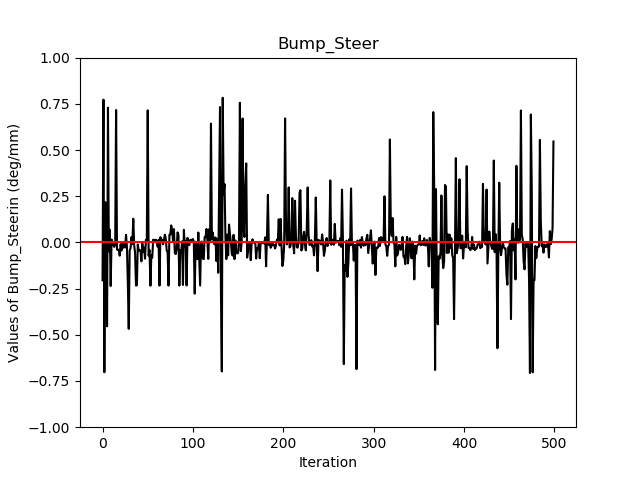}
\label{HEEDSBumpsteer1HP}}
\hspace*{-2.5em}
\subfigure{
\includegraphics[width=4.6cm,height=4.2cm]{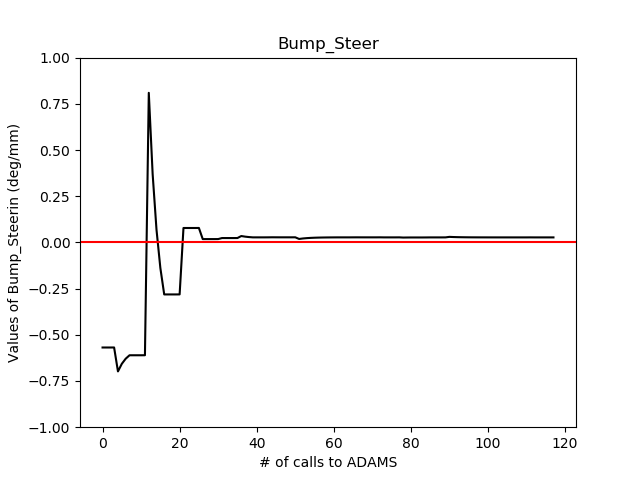}
\label{fminconBumpsteer1HP}}
\hspace*{-2.5em}
\subfigure{
\includegraphics[width=4.6cm,height=4.2cm]{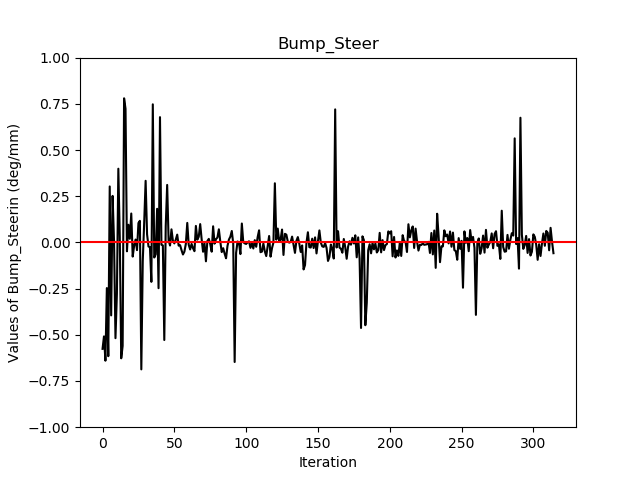}
\label{BOBumpsteer1HP}}
\caption{Performance of Bumpsteer in 1HP case (a) HEEDS.$\;$(b) nonlinear programming solver.$\;$(c) proposed approach.} 
\label{Bumpsteer1HP}
\hspace*{-3.1em}
\begin{minipage}{.5\textwidth}
  \centering
  \includegraphics[width=5.8cm,height=4.2cm]{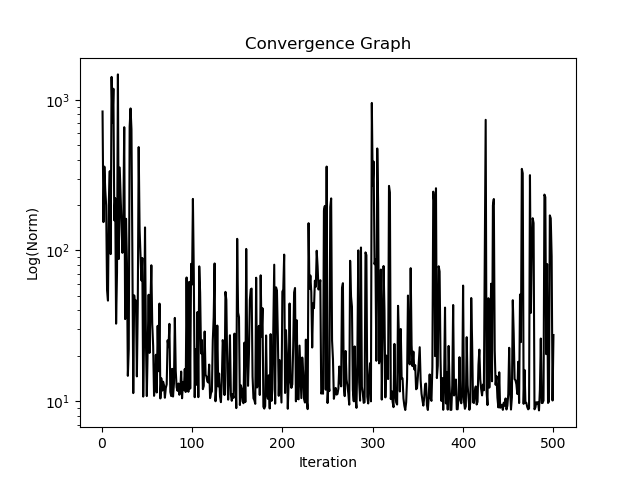}
  \captionof{figure}{Convergence (HEEDS) 2HP.}
  \label{ConvergenceHEEDS2HP}
\end{minipage}%
\begin{minipage}{.5\textwidth}
  \centering
  \includegraphics[width=5.8cm,height=4.2cm]{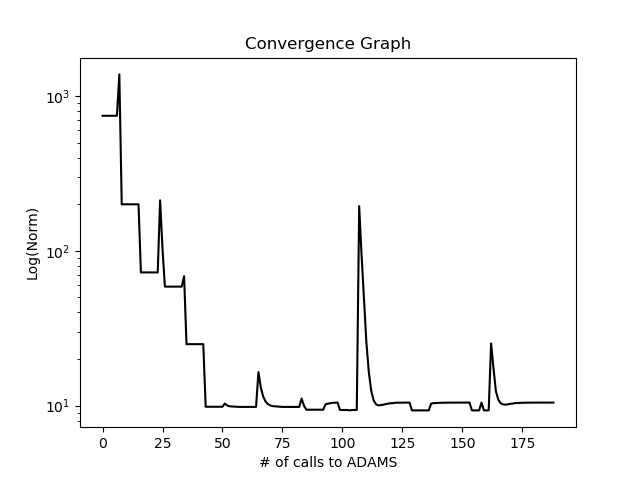}
  \captionof{figure}{Convergence (fmincon) 2HP.}
  \label{Convergencefmincon2HP}
\end{minipage}
\begin{minipage}{.5\textwidth}
  \centering
  \includegraphics[width=5.8cm,height=4.2cm]{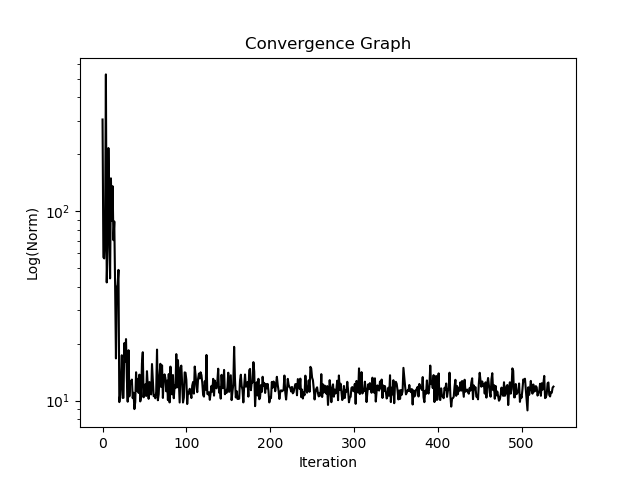}
  \captionof{figure}{Convergence (BO) 2HP.}
  \label{ConvergenceBO2HP}
\end{minipage}%
\begin{minipage}{.5\textwidth}
  \centering
  \includegraphics[width=5.8cm,height=4.2cm]{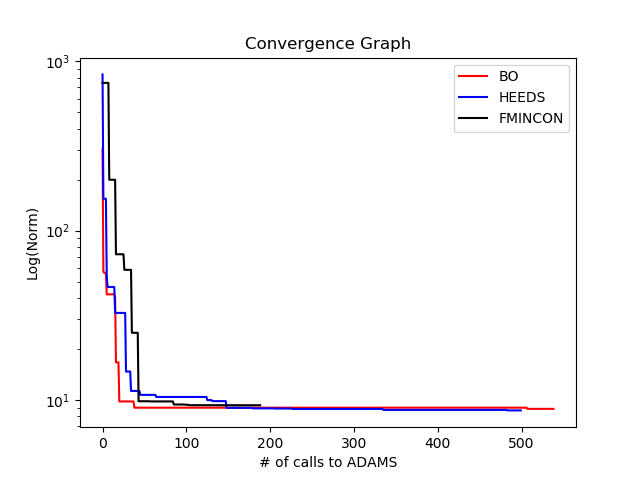}
  \captionof{figure}{Comparison of different convergence plots in 2HP case.}
  \label{Convergenceall2HP}
\end{minipage}
\end{figure}
\begin{figure}[htp]
\hspace*{-5.1em}
\subfigure{
\includegraphics[width=4.9cm,height=4.2cm]{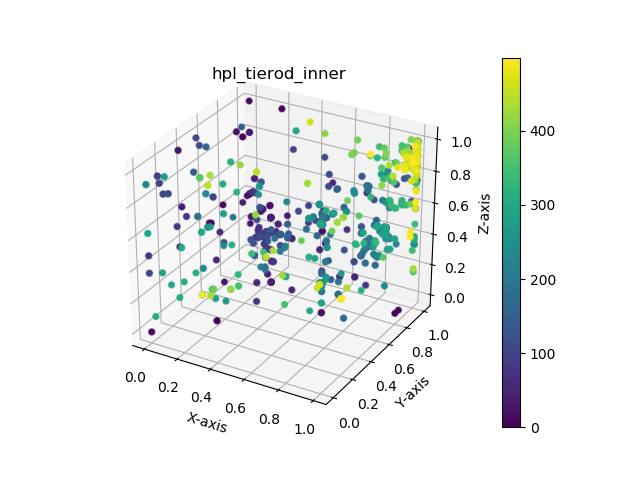}
\label{HEEDSScatterplotinner2P}}
\hspace*{-3.1em}
\subfigure{
\includegraphics[width=4.9cm,height=4.2cm]{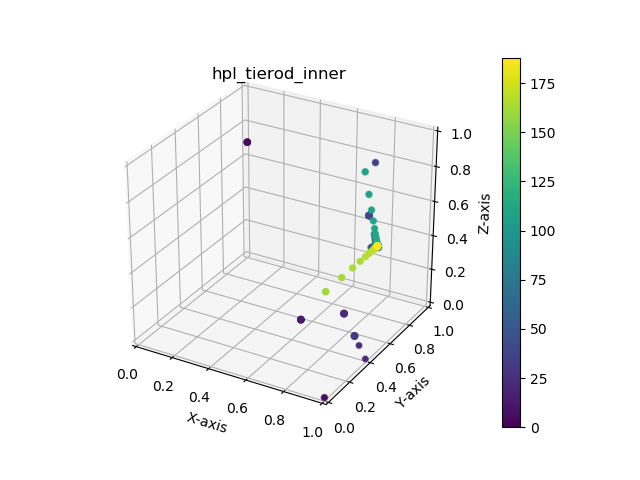}
\label{fminconScatterplotinner2HP}}
\hspace*{-3.1em}
\subfigure{
\includegraphics[width=4.9cm,height=4.2cm]{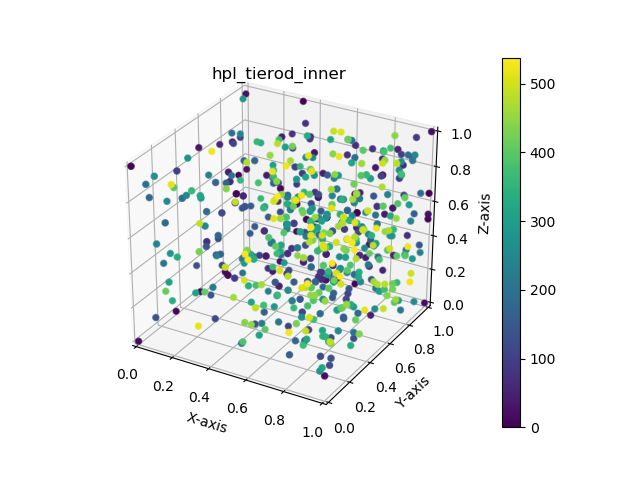}
\label{BOScatterplotinner2HP}}
\caption{Scatter plots for Inner Tie Rod Ball  in 2HP case (a) HEEDS.$\;$(b)nonlinear programming solver.$\;$(c) proposed approach.} 
\label{Scatterplotinner2HP}
\hspace*{-5.1em}
\subfigure{
\includegraphics[width=4.9cm,height=4.2cm]{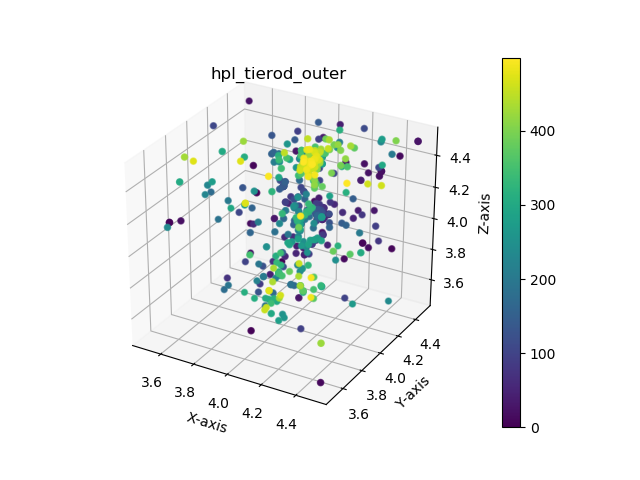}
\label{HEEDSScatterplotouter2P}}
\hspace*{-3.1em}
\subfigure{
\includegraphics[width=4.9cm,height=4.2cm]{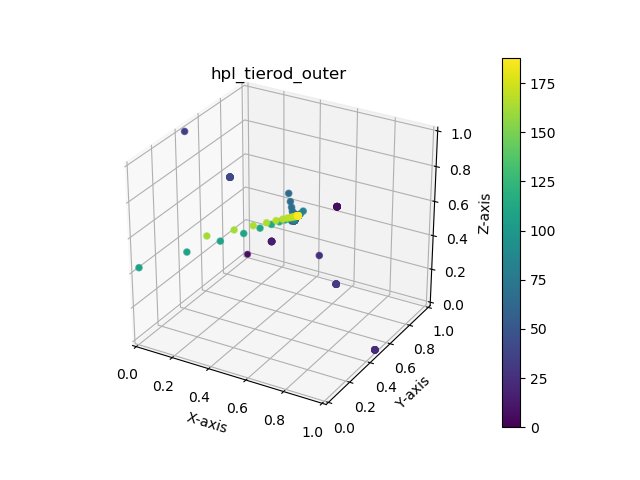}
\label{fminconScatterplotouter2HP}}
\hspace*{-3.1em}
\subfigure{
\includegraphics[width=4.9cm,height=4.2cm]{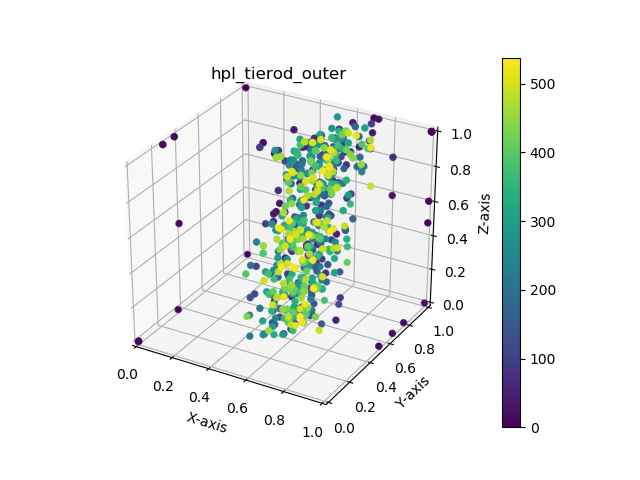}
\label{BOScatterplotouter2HP}}
\caption{Scatter plots for Outer Tie Rod Ball  in 2HP case (a) HEEDS.$\;$(b)nonlinear programming solver.$\;$(c) proposed approach.} 
\label{Scatterplotouter2HP}
\hspace*{-3.1em}
\subfigure{
\includegraphics[width=4.6cm,height=4.2cm]{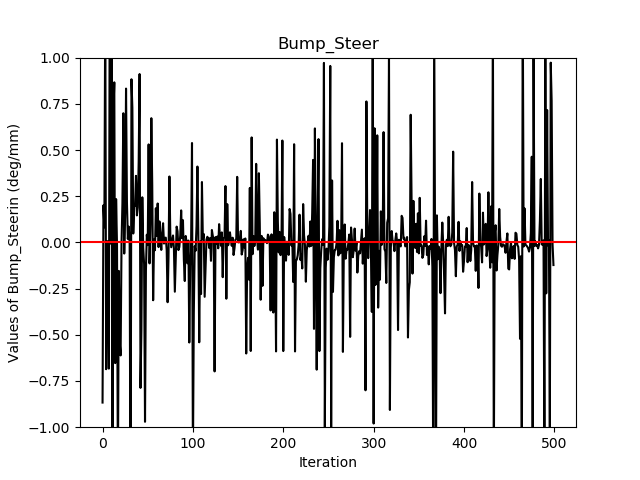}
\label{HEEDSBumpsteer2HP}}
\hspace*{-2.5em}
\subfigure{
\includegraphics[width=4.6cm,height=4.2cm]{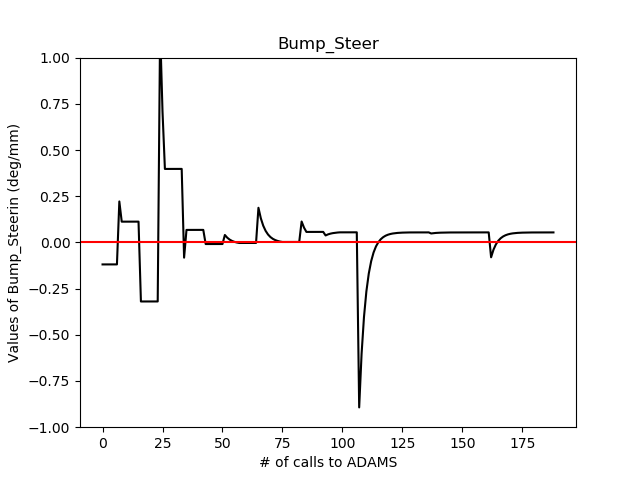}
\label{fminconBumpsteer2HP}}
\hspace*{-2.5em}
\subfigure{
\includegraphics[width=4.6cm,height=4.2cm]{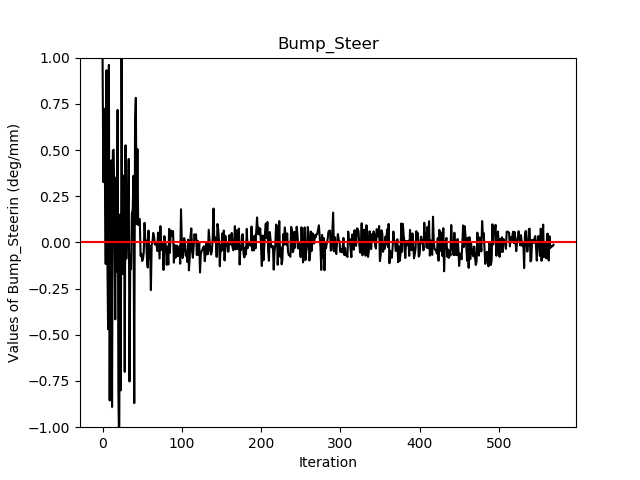}
\label{BOBumpsteer2HP}}
\caption{Performance of Bumpsteer in 2HP case (a) HEEDS.$\;$(b) nonlinear programming solver.$\;$(c) proposed approach.} 
\label{Bumpsteer2HP}
\hspace*{-3.1em}
\end{figure}

\subsection{Optimization of two hardpoints} 
The complexity in the design increases with the increase in the number of hardpoints to be optimized. We optimized the position of Inner Tie Rod Ball and Outer Tie Rod Ball in the second case-study.  The norm as a function of the iteration number for HEEDS, fmincon, and BO is shown in  Figs.~\ref{ConvergenceHEEDS2HP}, \ref{Convergencefmincon2HP}, and~\ref{ConvergenceBO2HP}, respectively. The convergence in the BO based approach is faster compared to HEEDS, as shown in Fig.~\ref{Convergenceall2HP}. The normalized scatter plots for each of the hardpoints is shown in Figs.~\ref{Scatterplotinner2HP} and~\ref{Scatterplotouter2HP}. 
Fig.~\ref{Bumpsteer2HP} shows the convergence to the target bump steer for all three methods. 
\section{Conclusion}
The performance of automotive design in terms of ride comfort and vehicle dynamics 
has a direct impact on passenger experience and customer satisfaction.
The suspension system is an important automotive system that carries the entire load of the vehicle and also responsible for a smooth ride. 
This paper primarily focuses on designing the geometry of the MacPherson front suspension system to satisfy desired kinematic performance metrics using the MSC ADAMS  Car model. 
Bayesian optimization is used to optimize the coordinate values of hard-points of the front suspension.  Conventionally, optimization of suspension design has been done using HEEDS software in the automobile industry. 
Our
results demonstrate that the proposed approach is able to optimize the parts of the suspension efficiently to enhance the comfort ride. The results of some case studies were described to highlight the performance of some commonly used algorithms with the proposed approach. It performs slightly better than the close source licensed HEEDS software commonly used in the automotive industry, 
and substantially better than 
optimization with a finite difference approximation to the gradient. Furthermore, we have two convergence criteria, based on the norm or on the acquisition function, which in contrast to HEEDS, allows us to determine when global convergence criteria are satisfied. 
\clearpage
\bibliographystyle{splncs04}
\bibliography{references}

\end{document}